\documentclass[twoside]{article}

\usepackage[accepted]{aistats2024}
\usepackage[utf8]{inputenc} 
\usepackage[T1]{fontenc}    
\usepackage{hyperref}       
\usepackage{url}            
\usepackage{booktabs}       
\usepackage{amsfonts, amsthm, amsmath, amssymb}       
\usepackage{nicefrac}       
\usepackage{microtype}      
\usepackage{xcolor}         
\newtheorem{theorem}{Theorem}
\usepackage{thmtools, thm-restate}
\usepackage{mathtools}
\usepackage[ruled,linesnumbered]{algorithm2e}
\usepackage{subfig}
\usepackage{listings}
\usepackage{thm-restate}
\usepackage{fp}
\usepackage[round]{natbib}
\definecolor{codegreen}{rgb}{0,0.6,0}
\definecolor{codegray}{rgb}{0.5,0.5,0.5}
\definecolor{codepurple}{rgb}{0.58,0,0.82}
\definecolor{backcolour}{rgb}{0.95,0.95,0.92}

\lstdefinestyle{mystyle}{
    backgroundcolor=\color{backcolour},   
    commentstyle=\color{codegreen},
    keywordstyle=\color{magenta},
    numberstyle=\tiny\color{codegray},
    stringstyle=\color{codepurple},
    basicstyle=\ttfamily\footnotesize,
    breakatwhitespace=false,         
    breaklines=true,                 
    captionpos=b,                    
    keepspaces=true,                 
    numbers=left,                    
    numbersep=5pt,                  
    showspaces=false,                
    showstringspaces=false,
    showtabs=false,                  
    tabsize=2
}

\hypersetup{
colorlinks,
citecolor=blue,
urlcolor=blue,
bookmarks=false,
hypertexnames=true}

\newcounter{myconjecture}
\newtheorem{conjecture}[myconjecture]{Conjecture}

\newcounter{mycorollary}
\newtheorem{corollary}[mycorollary]{Corollary}
\newtheorem{proposition}{Proposition}

\theoremstyle{definition}

\newtheorem{definition}{Definition}[section]
\newcounter{mylemma}
\newtheorem{lemma}[mylemma]{Lemma}
\newtheorem{remark}{Remark}

\newcommand{\localCyclic}{\mathbb{Z}_\mathcal{I}}
\newcommand{\localDihedral}{D_\mathcal{I}}
\newcommand{\localSymmetric}{S_\mathcal{I}}
\newcommand{\localFCyclic}[1]{\mathbb{Z}_{#1}}
\newcommand{\localFDihedral}[1]{D_{2 #1}}
\newcommand{\localFSymmetric}[1]{S_{#1}}

\begin{document}

%

%

\twocolumn[

\aistatstitle{A Unified Framework for Discovering Discrete Symmetries}

\aistatsauthor{Pavan Karjol \And Rohan Kashyap \And  Aditya Gopalan \And Prathosh A.P.}

\aistatsaddress{Department of Electrical Communication Engineering, \\ Indian Institute of Science, India} ]

\begin{abstract}
We consider the problem of learning a function respecting a symmetry from among a class of symmetries. We develop a unified framework that enables symmetry discovery across a broad range of subgroups including locally symmetric, dihedral and cyclic subgroups. At the core of the framework is a novel architecture composed of linear, matrix-valued and non-linear functions that expresses functions invariant to these subgroups in a principled manner. The structure of the architecture enables us to leverage multi-armed bandit algorithms and gradient descent to efficiently optimize over the linear and the non-linear functions, respectively, and to infer the symmetry that is ultimately learnt. We also discuss the necessity of the matrix-valued functions in the architecture. Experiments on image-digit sum and polynomial regression tasks demonstrate the effectiveness of our approach.
\end{abstract}

\section{Introduction}
It is well known that machine learning tasks often exhibit natural symmetries. As a result, the function to be learnt, say in a classification or regression setting, possesses additional structure in terms being invariant or equivariant to the underlying symmetry. Being able to exploit symmetry structure in the training pipeline confers benefits such as improved sample complexity, added explainability, fewer model parameters and improved generalizability. A classic case in which symmetry is leveraged is the convolutional neural network (CNN) architecture \citep{lecun1995convolutional} that intrinsically expresses equivariance to translations of input images in classification tasks.


A growing body of work has addressed the problem of incorporating known symmetries into the learning pipeline, either via augmenting data using the symmetry structure \citep{benton2020learning} or designing neural nets that inherently express functions with known symmetries \citep{zaheer2017deep, kicki2020computationally}. Consequently, it is known how to design architectures with $n$ inputs that are, say, invariant to arbitrary permutations of the input variables, or equivalently, neural functions that are $S_n$-invariant where $S_n$ is the group of permutations on $n$ elements \citep{dummit2004abstract}.

However, there are often settings in which the target function possesses a symmetry which is a priori {\em unknown}, but known to belong to a class of possible symmetries (subgroups of $S_n$). We are interested in the problem of discovering such an unknown symmetry automatically from data. Consider, for instance, data representing measured states of a system of multiple particles (e.g., positions, velocities, etc.), with the target function representing a physical quantity of interest depending on the state, such as potential energy. If only $k$ of the $n$ particles (whose identities are unknown) actually interact with each other (maybe because they are the only charged particles), then the net energy is invariant to permutations of the positions of this subset of particles alone. Here, the target function exhibits invariance with respect to the subgroup of permutations $S_k$ associated to the position indices of these $k$ particles, which are not known upfront. On the other hand, the system's kinetic energy is unchanged under permutations of the subset of velocity parameters of the system state. In general, when the semantics of the target function and/or the input variables are unknown, then so is the underlying symmetry. A similar problem arises in computer vision as that of learning a classifier that can detect patterns or objects in an image while being invariant to local transformations or symmetries applied to specific regions or parts of the image \citep{lazebnik2004semi,felzenszwalb2009object}. 

We consider the problem of learning a function $f: X \rightarrow Y$, given data $\big \{ \left(x^{(u)}, y^{(u)} \right) \big \}_{u=1}^m$ and a collection of non-trivial subgroups\footnote{Restricting to subgroups of $S_n$ is justified by the fact that any finite group is isomorphic to a subgroup of $S_{n}$ for some $n$ by Cayley's theorem \citep{dummit2004abstract}.} of $S_n$, one of which $f$ is invariant with respect to (i.e., $f \circ g \equiv f$ for every transformation $g$ in some subgroup of $S_n$). %
For a sufficiently rich collection of possible symmetry subgroups\footnote{In general, if we consider {\em all} subgroups of $S_n$, then the problem of learning a specific symmetry is known to be computationally intractable \citep{ensign2020complexity}.}, we provide a unified and easy-to-use framework comprising of a parametric architecture together with algorithms to tune it and learn the underlying symmetry (subgroup). Our specific contributions are presented in the following subsection.

\subsection{Contributions}
\begin{itemize}
    \item We introduce a general framework for discovering a variety of discrete symmetries. Our framework allows for efficiently learning functions that can be invariant to {\em any} locally symmetric, dihedral or cyclic subgroup using the same architecture.     
    
    \item The unified architecture that forms the backbone of our framework is comprised of a novel combination of (learnable) linear, matrix-valued and non-linear functions. We explicitly characterize the structure of both these transformations, in particular showing how they correspond to a variety of subgroups. To the best of our knowledge, this is the first unified framework to discover a wide range of discrete symmetries.
    
    \item Leveraging the specific structure of the linear transformations in our unified architecture, we devise an efficient training algorithm based on multi-armed bandits (for discrete optimization over matrices representing the learnable linear part) along with stochastic gradient descent (for continuous optimization over the nonlinear part). The bandit sampling allows for efficient search across the entire family of matrices associated to various symmetries, and, with our structural characterization, allows for interpretable results. 
\end{itemize}

Note that, the goal of our paper is to propose a unified architecture for the discovering the underlying discrete subgroup. Thus, we argue that after the discovery of the \textit{correct} symmetry using our framework, one could in practice utilize any off-the-shelf models \citep{kicki2020computationally, zaheer2017deep, yang2023generative} to improve the model accuracy.

\subsection{Related Work}
\subsubsection{Group Equivariance} The utilization of symmetries in deep learning has garnered significant research interest in recent years \citep{bronstein2021geometric, dehmamy2021automatic}. Within this context, \cite{cohen2016group} introduced $G$-equivariant neural networks as an extension of Convolutional Neural Networks (CNNs) to encompass a broader range of symmetries. Furthermore, \cite{kondor2018generalization} establish convolution formulae in a more general setting, i.e., invariance under the action of any compact group and \cite{cohen2019general} delve into the application of $G$-CNNs on homogeneous spaces using equivariant linear maps. 

\subsubsection{Discrete Groups} 
The study of invariance to finite groups has received considerable attention in the existing literature.  \cite{kicki2020computationally} proposed an approach that utilizes invariant polynomials to design $G$-invariant neural networks  $f : X \rightarrow \mathbb{R}$, where $X$ is a compact subset of $R^n$, achieved through a combination of a $G$-equivariant transformation block and the sum-product layer. They demonstrate the universality of their approach for larger and hierarchical subgroups of $S_{n}$. In a different approach, \cite{zaheer2017deep} introduced permutation-equivariant functions defined on sets using a decomposable representation expressed as $\rho \left( \sum_i \phi \left (x_i \right) \right)$. Motivated by these, we consider invariance under the action of subgroups of $ G \leq S_{n}$, when the underlying subgroup is unknown.

\subsubsection{Automatic Symmetry Discovery} \cite{dehmamy2021automatic} presents a Lie algebra convolution network (L-conv) for constructing feedforward architectures that exhibit equivariance to arbitrary continuous groups. \cite{benton2020learning} propose a different approach by parameterizing a distribution over training data augmentations, while \cite{zhou2020meta} introduce a meta-learning framework that addresses symmetries through the reparameterization of network layers. Building upon the idea of establishing invariant symmetry-adapted data representations, \cite{anselmi2019symmetry} investigates the use of regularization on the representation matrix for unsupervised orbit learning.

Recently \cite{yang2023generative} proposed LieGAN, which is based on generative adversarial approach to discover the underlying subgroup. However, most of the existing methods emphasize on continuous group symmetries. In this work, we propose a similar solution for discrete group symmetries. In particular, we demonstrate that a unified architecture can be used for arbitrary symmetry discovery ($\{ \localCyclic, \localDihedral, \localSymmetric \}$) using a multi-armed bandits setting which aids in identifying the exact symmetry learned as discussed in Section \ref{proposed framework} and \ref{Experiments} respectively.

\section{Proposed Method}
\label{proposed framework}
\subsection{Mathematical Preliminaries}
The group $S_n$ is the set of all permutations on $n$ elements along with the natural group multiplication (composition) and inverse operations. By a {\em symmetry} we mean a subgroup $G \leq S_{n}$; all groups used henceforth are assumed to be of this form. The group generated by an element $g$ is $\langle g \rangle = \{g, g^2, g^3, \ldots \}$. We use $f \circ g$ to denote function composition: $(f \circ g) (x) = f(g(x))$.

\begin{definition}
    Let $\mathcal{I} = \{i_1, \ldots, i_k \} \subset [n]$ be an index set with $i_1 < \cdots < i_k$. 
    \begin{itemize}
        \item $\localCyclic$ is the locally cyclic group corresponding to $\mathcal{I}$, generated by the permutation $\pi \in \localFSymmetric{n}$ such that $\pi(i) = i_{\tau(j)}$ if $i = i_j$ and $\pi(i) = i$ otherwise. Here, $\tau(j) = (j \mod n) + 1$ denotes the cyclic shift operator. 
        \item $\localDihedral$ is the locally dihedral group corresponding to $\mathcal{I}$, defined as $\{\pi, \pi^2, \ldots, \sigma \pi, \sigma \pi^2, \dots \}$, where $\pi \in \localFSymmetric{n}$ is as defined above and $\sigma \in \localFSymmetric{n}$ is defined by $\sigma \left( i_l \right) = \sigma \left( i_{k-l+1} \right)$  $\forall l \in [k]$ (reflection about the center of $\mathcal{I}$).
        \item $\localSymmetric$ is the locally symmetric group corresponding to $\mathcal{I}$, consisting of all permutations that move elements only within $\mathcal{I}$, i.e., $\localSymmetric = \{ \pi \in S_n: \pi(j) = j \, \, \forall j \notin \mathcal{I} \}$.
        \item $\localFCyclic{k} = \localCyclic$; $\quad$   $\localFDihedral{k} = \localDihedral $; $\quad$  $\localFSymmetric{k} = \localSymmetric$  with $\mathcal{I} = [k]$ (the first $k$ elements of  $[n]$).
    \end{itemize}
    \label{def:Z_I}
\end{definition}

 \begin{definition} 
Let $g \in S_n$. The action of $g$ on $\mathbb{R}^n$ is the map $x \mapsto g\cdot x$ given by $(g \cdot x)_i = x_{g(i)}$ $\forall i \in [n]$.
\end{definition}

\begin{definition} 
The orbit of $x \in X$ under the action of group $G$ is  defined as $\mathcal{O}_G(x) = \{g \cdot x| g \in G \}$.
\end{definition}

\begin{definition} 
A function  $f : X \rightarrow \mathbb{R}$ is said to be $G$-invariant, if $f(x) = f(g \cdot x), \forall g \in G, x \in X$. 
\end{definition}

\begin{definition}
Let $X, Y \subseteq \mathbb{R}^n$. A function  $f : X \rightarrow Y$ is said to be $G$-equivariant, if for any $g \in G$,  $\exists$ $\Tilde{g} \in G$,  $f(g \cdot x) = \Tilde{g} \cdot f( x),   \forall x \in X$.
\end{definition}

\subsection{Problem statement}
Let $X = [0,1]^n \setminus E$  denote the input (instance) domain, where $E = \{[x_1, x_2 \dots , x_n]^T \in [0,1]^n : x_i = x_j \text{ for some } i, j \in [n] \text{ with } i \neq j \}$. Note that the $n$-dimensional measure of the set $E$ is zero. We frame the symmetry discovery problem as follows:\\ Given data $\big \{ \left(x^{(u)}, y^{(u)} \right) \big \}_{u=1}^m$ with $x^{(u)} \in X, y^{(u)} \in \mathbb{R}$, and the collection of non-trivial subgroups  $\mathcal{G} = \cup_{\mathcal{I} \subseteq [n], |\mathcal{I} | > 1} \{ \localCyclic, \localDihedral, \localSymmetric \}$, we aim to learn a function $f: X \to \mathbb{R}$ such that $f$ is $G$-invariant for some $G \in \mathcal{G}$ with respect to the data. 
Specifically, we wish to efficiently solve the following empirical risk minimization (ERM) problem,
\begin{equation}
\label{eq:ERM}
    \arg\min_{f \in \mathcal{F}(\mathcal{G})} \frac{1}{m} \sum_{u=1}^m \ell \left(y^{(u)}, f\left(x^{(u)} \right) \right),
\end{equation}
where the hypothesis class $\mathcal{F}(\mathcal{G})$ is comprised of all functions that are $G$-invariant for some $G \in \mathcal{G}$, i.e., $\mathcal{F}(\mathcal{G}) = \{f: X \to \mathbb{R}: \exists G \in \mathcal{G} \text{  s.t. $f$ is $G$-invariant} \}$, and $\ell$ stands for a loss function such as squared or absolute error loss.


\begin{figure}[htp!]
    \centering
    \includegraphics[width= \columnwidth, keepaspectratio, draft=false]{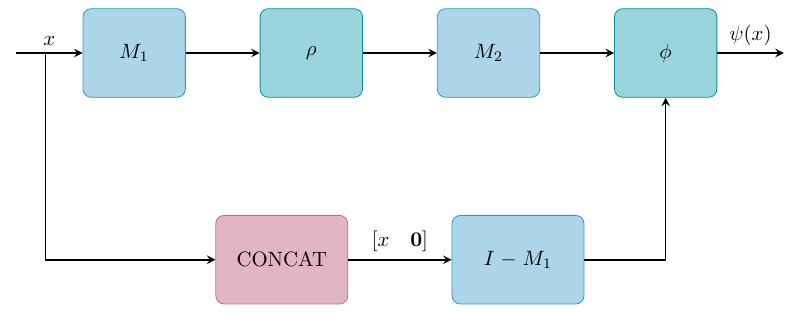}
    \caption{Proposed unified architecture for discovering symmetries, composed of linear transformations ($M_1$, $M_2$), matrix-valued ($\rho$) and non-linear function ($\phi$). $\rho$ is explicitly fixed whereas $M_1, M_2$ and $ \phi$ are trainable. Theorem \ref{Main framework} guarantees that the architecture can express functions invariant to any locally symmetric, dihedral and cyclic. Here, $\phi$ is represented by a neural network and trained using gradient descent while $M_1, M_2$ are optimized using bandit sampling over a discrete space of matrices.}
    \label{fig:diagram}
\end{figure}

\subsection{Proposed framework}
We aim to develop a framework for solving the symmetry discovery problem defined above in the problem statement. It is not a priori clear how to efficiently search over the function class $\mathcal{F}(\mathcal{G})$ -- observe that $\mathcal{G}$ is an exponentially large (in $n$) set of subgroups. 

Our solution strategy is based on finding a standard decomposition for any function $\psi$ in the function class $\mathcal{F}(\mathcal{G})$. To this end, we first consider each type of subgroup individually and prove a structural decomposition of the form $\psi = \phi \circ \rho$ for any $\psi$ which is invariant to that group. We then design a single decomposition of the form $\phi \circ M_2 \circ \rho \circ M_1$ that effectively integrates all the individual decompositions. 

Our first result shows that any $\localFCyclic{k}$-invariant function can be expressed as a composition of an $\localFSymmetric{k}$-invariant function and a specific matrix-valued function.

\begin{restatable}{theorem}{thmZkSk}
\label{Z_k:S_k}
    Let $\psi: [0,1]^k \rightarrow \mathbb{R} $ be $\localFCyclic{k}$-invariant. There exists an $\localFSymmetric{k}$-invariant function $\phi : [0,1]^{k \times 2} \rightarrow \mathbb{R}$ and $\rho: [0,1]^k \rightarrow [0,1]^{k \times 2}$,  such that 
    \begin{equation}
        \label{thm1:e1}
        \psi = \phi \circ \rho,
    \end{equation}
    where $\rho$ is defined as, 
    \begin{align}
        \begin{bmatrix}
        x_1 \\
        \vdots \\
        x_k
        \end{bmatrix} 
        &\mapsto
        \begin{bmatrix}
        x_1 &x_2 \\
        x_2 &x_3 \\
        \vdots &\vdots\\
        x_k &x_1
        \end{bmatrix}
        \label{rho def Z_k}
    \end{align}
    
\end{restatable}

\textit{Proof. (Sketch)} The  $\localFCyclic{k}$-invariant function $\psi$ must assign the same value to every element of any $\localFCyclic{k}$-orbit. We show that any such orbit $\mathcal{O}_{\localFCyclic{k}}(x)$ can be uniquely associated with the corresponding $\localFSymmetric{k}$-orbit $\mathcal{O}_{\localFSymmetric{k}}(\rho(x))$. From this, it follows that by defining the $\localFSymmetric{k}$-invariant function $\phi$ to take the same value across any orbit of the form $\mathcal{O}_{\localFSymmetric{k}}(\rho(x))$ as $\psi$ does across the orbit $\mathcal{O}_{\localFCyclic{k}}(x)$ (and an arbitrary value across orbits not of the form $\mathcal{O}_{\localFSymmetric{k}}(\rho(x))$), we obtain the result.

We also assess the regularity conditions such as smoothness ($C^{\infty}$) and continuity ($C^0$) of the $\psi$ and $\phi$ function, and in this regard we state the following theorem.
\begin{theorem}
    \label{regularity}
    Under the same hypothesis of Theorem \ref{Z_k:S_k}, the $\phi$ function is smooth ($C^\infty$) whenever $\psi$ function is $C^\infty$. Similarly, the $\phi$ function is continuous ($C^0$) whenever $\psi$ function is $C^0$.
\end{theorem}

We state the following lemma, to prove Theorem \ref{regularity}.
\begin{lemma}
    \label{diffeomorphism}
    \textit{The matrix-valued function $\rho$ defined in \eqref{rho def Z_k} is a diffeomorphism between $[0,1]^k$ and its image $\rho([0,1]^k)$.}
\end{lemma}
The proof for Lemma \ref{diffeomorphism} is given in the Appendix section.

\begin{proof}
From \ref{thm1:e1}, we have $\psi = \phi \circ \rho$. Thus, $\psi \circ \rho^{-1} = \phi$. From Lemma \ref{diffeomorphism}, $\rho^{-1}$ is smooth ($C^{\infty}$) since $\rho$ is a diffeomorphism. Thus, if $\psi$ is a continuous function ($C^0$), then $\phi$ is composition of $C^{\infty}$ function with a $C^0$ function which in turn implies composition of two $C^0$ functions. Thus $\phi$ is $C^0$. Similarly, if $\psi$ is $C^\infty$, then $\phi$ is a composition of two $C^\infty$ functions. Thus $\phi$ is $C^\infty$.
\end{proof}
Results of the same form as Theorem \ref{Z_k:S_k} and Theorem \ref{regularity} hold for $\psi$ being a $\localFDihedral{k}$- or $\localFSymmetric{k}$-invariant function by replacing the definition of the function $\rho$ with the appropriate definition in Table \ref{tab:rho and invariance}.

         
         

\begin{table}[htp]
    \scriptsize
    \centering
    \begin{tabular}{c|c|c|c}
         &$\localFSymmetric{k}$ & $\localFCyclic{k}$ &$\localFDihedral{k}$\\
        \hline
        \hline
$\rho(x)$  
         &$\begin{bmatrix}
         \vdots \\
             x_i \quad x_i \\
         \vdots \\            
         \end{bmatrix}_{i\in [k]}$
         
         &$\begin{bmatrix}
         \vdots \\
             x_i \quad x_{\tau(i)} \\
         \vdots \\              
         \end{bmatrix}_{i \in [k]}$ 
         
         &$\begin{bmatrix}
         \vdots \\
             x_i \quad x_{\tau(i)} \\
             x_{\tau(i)} \quad x_i \\
         \vdots \\            
         \end{bmatrix}_{i \in [k]}$ \\
\hline         
    \end{tabular}
    \caption{Subgroups of $\localFSymmetric{n}$ and corresponding definitions of the matrix-valued function $\rho$, where $\tau$ is cyclic right shift by 1 element.}
    \label{tab:rho and invariance}
\end{table}

We now state our main result,  which is a {\em single} canonical functional decomposition that includes functions invariant to all the subgroups of type $\localCyclic$, $\localSymmetric$ and $\localDihedral$, in  Theorem \ref{Main framework}. The key idea is to introduce `selection' matrices that appropriately reduce a general function to the specific type of subgroup as in Theorem \ref{Z_k:S_k} ($\localFCyclic{k}$, $\localFDihedral{k}$ or $\localFSymmetric{k}$).


 \begin{theorem}[Unified symmetry discovery framework]
 \label{Main framework}
Let $\mathcal{B}$ denote the class of all functions from $X \to \mathbb{R}$ of the form: 
$$
x \mapsto \phi \left(\left[\begin{array}{l}
\left(M_2 \circ \rho \circ M_1\right)(x) \\
\left(I-M_1\right) \left([x \quad  \bf{0}]\right) 
\end{array}\right]\right)
$$
where,
\begin{itemize}
\item $M_1$ and $M_2$ are matrices of size $n \times n$ and $n^2 \times n^2$ respectively.
\item  $\phi:[0,1]^{n(n+1) \times 2} \rightarrow \mathbb{R}$ is an $\localFSymmetric{n^2}$-invariant function where  the invariance pertains to the initial $n^2$ rows out of a total of $n(n+1)$,
and 
\item $\rho: X \to [0,1]^{n^2 \times 2}$ is a matrix-valued function given as, 
$ \begin{bmatrix}
    x_1 \\ x_2 \\ \vdots \\ x_n
\end{bmatrix} \mapsto \begin{bmatrix}
    \vdots  \\
    x_i \quad x_j \\
    \vdots \\
\end{bmatrix}_{i,j \in [n]}$.
\end{itemize}

 Let $\mathcal{I} = \{i_1, i_2, \dots i_k\} \subseteq [n]$ $(k>1)$ and $\tau$ be the permutation (cyclic shift) as defined in \ref{def:Z_I}. Then, the following hold:
 \begin{itemize}
 
     \item[a)] Any $\localSymmetric$-invariant function belongs to $\mathcal{B}$. Moreover, the matrices $M_1$ and $M_2$ in its decomposition have the forms:
        \begin{align}
        M_1[u,v] &= \begin{cases}
            1, \quad \text{if } u \in [k] \text{ and }v = i_u \\
            0, \quad \text{otherwise}.
        \end{cases}  \label{common M1} \\
        M_2[u,v] &=  \begin{cases}
            1, \quad \text{if } u \in [k^2], \quad u = v \text{ and } \\
            \: \quad   \left(\rho \circ M_1 \right)(x)[v] = (x_i, x_i) \\
            \: \quad \text{ for some } i \in \mathcal{I} \\
            0, \quad \text{otherwise}.
        \end{cases} \label{eq:M2_Sk}
        \end{align}
        
    \item[b)] Any $\localCyclic$-invariant function belongs to $\mathcal{B}$. Moreover,   $M_1$ is of the form as given in \eqref{common M1} and $M_2$ is as follows:
        \begin{align}
        M_2[u, v] &= \begin{cases}
        1, \quad \text{if } u \in [k] \text{ and }\\
        \: \quad  \left(\rho \circ M_1 \right)(x)[v] = (x_{i_u},x_{\tau(i_u)}) \\
        0, \quad \text{otherwise}.\\
        \end{cases} \label{eq:M2_Zk}
        \end{align}

    \item[c)] Any $\localDihedral$-invariant function belongs to $\mathcal{B}$. Moreover,   $M_1$ is of the form as given in \eqref{common M1} and $M_2$ is as follows:
        \begin{align}
        M_2[u, v] &= \begin{cases}
        1, \quad \text{if } u \in [k] \text{ and } \\
        \: \quad  \left(\rho \circ M_1 \right)(x)[v] = (x_{i_u},x_{\tau(i_u)}) \\
        1, \quad \text{else if } u \in [2k] \setminus [k] \text{ and } \\
        \: \quad  \left(\rho \circ M_1 \right)(x)[v] = (x_{\tau(i_{u-k})}, x_{i_{u-k}}) \\
        0, \quad \text{otherwise}.\\
        \end{cases} \label{eq:M2_D2k}
        \end{align}              

 \end{itemize}
\end{theorem}
\textit{Proof. (Sketch)} 
The goal is to show that $\phi \circ  M_2 \circ \rho \circ M_1$ (with $\phi$ being $\localFSymmetric{n^2}$-invariant and $\rho$ is as defined in the Theorem \ref{Main framework}) is equivalent to $\phi \circ \rho$ (with $\phi$ being $\localFSymmetric{k}$-invariant and $\rho$ is specific to the unknown subgroup, an example of which is given in Theorem \ref{Z_k:S_k}). This is achieved via appropriately choosing $M_1$ and $M_2$ so that the elements of the form $(x_i, x_j)$ specific to the subgroup are selected. The $M_1$ helps in selecting appropriate indices over which the subgroup acts and $M_2$ helps in identifying the broader category (symmetric, cyclic or dihedral) of the subgroup.

\begin{remark}
While the domain of the function $\phi$ is defined as $[0,1]^{n(n+1) \times 2}$, it is worth noting that, when $\phi$ is post-composed with the transformation $M_2 \circ \rho \circ M_1$, the input to $\phi$ inevitably contains zeros at specific positions, which are contingent upon the selection matrices $M_1$ and $M_2$. Consequently, the $\localFSymmetric{n^2}$-invariance exhibited by $\phi$ effectively translates to permutation invariance with respect to the remaining indices (among the first $n^2$), namely the non-zero elements. Further elucidation on this aspect is presented in the Appendix section of this paper.
\end{remark}


We further remark that Theorem \ref{Main framework} can be extended to express functions invariant to wider classes of subgroups. The following results offer a glimpse of how this can be achieved, for instance, for product groups.

\begin{restatable}[Invariance to product groups]{theorem}{thmProductGroups}
    Let $[n] = \bigcup\limits_{j=1}^L \mathcal{I}_j$ be a partition of $[n]$, $G_i \in \{S_{\mathcal{I}_j}, D_{\mathcal{I}_j}, \mathbb{Z}_{\mathcal{I}_j}\}, \forall j \in [L]$ and $G = G_1 \times G_2 \times \cdots G_L$ such that no two groups $G_i, G_j$  are isomorphic and only one of the component groups is of the type $\localSymmetric$.
    Let $\psi$ be a $G$-invariant function, then there exists an $\localFSymmetric{l}$-invariant function $\phi$ and a specific matrix-valued function $\rho$, such that,
    \begin{equation}
        \psi = \phi \circ \rho.
    \end{equation}
    \label{Product group:S_l}
\end{restatable}

\textit{Proof. (Sketch)} Let us define the function $\rho$, which maps to the appropriate elements of the form $(x_i, x_j)$, corresponding to individual components of the product group $G$. It is important to note that $\rho$ is both injective and $G$-equivariant. We denote the variable $l$ (as in $\localFSymmetric{l}$-invariant function) to represent the total number of these appropriate elements. With this setup, we can demonstrate that each $G$-orbit can be uniquely associated with an $S_{n^2}$-orbit within the transformed space denoted as $Im(\rho)$. This mapping is analogous to the proof technique employed in Theorem \ref{Z_k:S_k}.

\begin{corollary}
    Let $\sigma \in S_n$ and $G = \big\langle \sigma \big\rangle$ such that whose disjoint cycles have unique lengths. Let $\psi$ be a $G$-invariant function, then there exists an $\localFSymmetric{l}$-invariant function $\phi$ and a specific matrix-valued function $\rho$, such that, $\psi = \phi \circ \rho.$        
    \label{generated subgroup}
\end{corollary}
\textit{Proof.} We use the fact that any permutation $\sigma$ can be decomposed into disjoint cycles. Hence $G = \mathbb{Z}_{\mathcal{I}_1} \times \mathbb{Z}_{\mathcal{I}_2} \cdots \times \mathbb{Z}_{\mathcal{I}_L}$ with no two $\mathbb{Z}_{\mathcal{I}_k}, \mathbb{Z}_{\mathcal{I}_l}$ are isomorphic (because the lengths are different). Applying Theorem \ref{Product group:S_l}, we prove the claim. 
\subsection{Optimization for discovering symmetries}
\label{ss:optimization}

Having proposed, via Theorem \ref{Main framework}, a common functional form ($\phi \circ M_2 \circ \rho \circ M_1$) for any function invariant to symmetries of type $\localCyclic$, $\localDihedral$ or $\localSymmetric$, we turn to methods to fit the functional form to data and discover the underlying symmetry.

A straightforward approach is to employ standard stochastic gradient descent (SGD)-type optimization jointly over $\phi$, parameterized as a neural network, and $M_1, M_2$, parameterized as matrices in $\mathbb{R}^{n \times n}$ and $\mathbb{R}^{n^2 \times n^2}$, respectively. However, in view of the discrete structure of $M_1, M_2$ prescribed explicitly by Theorem \ref{Main framework} (equations \eqref{common M1}-\eqref{eq:M2_D2k}), we resort to multi-armed bandit sampling to learn the best $(M_1,M_2)$ pair in an `outer loop', with SGD over $\phi$ running in the `inner loop'. Specifically, each arm of the bandit corresponds to a $(M_1,M_2)$ pair, and the reward for it is the negative of the loss that SGD over $\phi$ obtains for that pair. This approach is advantageous for two reasons: (i) It confers interpretability in the sense that the underlying symmetry can be directly read off from the  $M_1, M_2$ which is ultimately learnt by the bandit outer loop, (ii) A bandit algorithm over $(M_1,M_2)$ performs global optimization and avoids the potential pitfalls of using gradient descent that could get stuck in local optima. 

{\bf Linear Thompson Sampling (LinTS)-based bandit optimization algorithm:} Observe that although the space of matrices $(M_1, M_2)$ guaranteed by Theorem \ref{Main framework} is discrete, it is still an exponentially large set. To enable efficient search over this set, we resort to using the linear parametric  Thompson sampling algorithm (LinTS) \citep{agrawal2013thompson}. In this strategy, whose pseudo code appears in Algorithm \ref{alg:linTS}, each possible pair of matrices $(M_1,M_2)$, denoting an arm of the bandit, is represented uniquely by a {\em binary} feature vector of an appropriate dimension $d$ (described in detail below). The reward from playing an arm with feature vector $a$ (which is the negative loss after optimizing for $\phi$ using SGD) is assumed to be linear in $a$ with added zero-mean noise, i.e., $\exists \mu^\star \in \mathbb{R}^d$ such that the expected reward upon playing $a$ is $a^{\top}\mu^{\star}$. LinTS maintains and iteratively updates a (Gaussian) probability distribution (lines \ref{line:Bupdate}, \ref{line:fupdate} and \ref{line:muupdate}) over the unknown reward model $\mu^{\star}$, and explores the arm space by sampling from this probability distribution in each round (line \ref{line:sampling}). 

Using LinTS for exploring across $(M_1, M_2)$ is advantageous for several reasons. The chief one is that even though the arm set of binary vectors, representing all possible $M_1, M_2$ matrices, is exponentially large (of cardinality $O(3\cdot 2^n)$), finding the arm maximizing the reward for a sampled vector $\mu$ (line \ref{line:opt}) is a constant-time operation. Another reason to prefer LinTS as a search strategy is that it enjoys a rigorous guarantee on the probability of error in finding the best arm in a true linear model, as we show in Theorem \ref{thm:TS} below.
\RestyleAlgo{ruled}
\begin{algorithm}[hbt!]
\caption{Linear Parametric Thompson Sampling for Subgroup Discovery}\label{alg:two}
\label{alg:linTS}

$\textbf{Initialize:}$ $\mathcal{A} \subset \{0,1\}^d$ (arm set: binary feature vectors representing each pair of matrices $(M_1,M_2)$), \\
$B \leftarrow I_d$ (prior covariance), \\
$f \leftarrow 0 \in \mathbb{R}^d, \hat{\mu} \leftarrow 0 \in \mathbb{R}^d$ (prior mean),\\
$\nu > 0$ (variance inflation parameter),\\
$T$ (time horizon). \\
\For{$t \in \big \{1, 2, \ldots, T \big\}$}{
    Sample $\mu$ independently from $\mathcal{N}\left(\hat{\mu}, \nu^2 B^{-1}\right)$ \label{line:sampling} \\
    $a \leftarrow \arg \max_{a' \in \mathcal{A}}  \mu^{\top} a'$ \label{line:opt}  \\
    $B \leftarrow B+ a a^{\top}$ \label{line:Bupdate} \\
    Fix matrices $M_1, M_2$ in the architecture as per $a$, and run SGD over $\phi$ with loss function $L(\phi) = \frac{1}{m} \sum \limits_{u=1}^m \ell \left(y^{(u)}, (\phi \circ M_2 \circ \rho \circ M_1) \left(x^{(u)} \right) \right)$ to obtain $\tilde{\phi}$ \\
    Set reward from arm $a$: $\gamma \leftarrow -L(\tilde{\phi})$ \\
    $f \leftarrow f+ a \gamma$ \label{line:fupdate}\\
    $\hat{\mu} \leftarrow B^{-1}f$ \label{line:muupdate}\\ 
  }
\textbf{return} $A_T = \arg \max_{a \in \mathcal{A}} a^{\top} \hat{\mu}$ (best arm for the estimated linear model) \label{line:reco}
\end{algorithm}


\textbf{Features for bandit arms:} 
To specify the feature vector for each bandit arm, we employ one-hot encoding to represent the general subgroup category in the order given as, locally symmetric, dihedral, and cyclic respectively. An n-dimensional vector is utilized to represent the corresponding indices, where the indices pertaining to the subgroup category are set to 1, while the remaining indices are set to 0. Subsequently, this vector can be concatenated with a one-hot encoded representation of the subgroup category. For example, with $n = 10$, $G = \localCyclic$, and $\mathcal{I} = \{3, 5, 6, 8\}$ the overall feature vector is given as follows:
\begin{equation*}
    a = [{\color{blue}0, 0, 0, 1, 0, 1, 1, 0, 1, 0,\;} {\color{red} 0, 0, 1}]^T.
\end{equation*}
The first $n$ indices (in blue) above correspond to the actual indices, while the last three indices (in red) indicate the respective subgroup type.

Our next result is a performance guarantee for the LinTS algorithm (Algorithm \ref{alg:linTS}), showing a bound on its probability of misidentifying the optimal arm in a linear reward model. 

\begin{restatable}[Error probability bound for LinTS]{theorem}{thmTS}
    \label{thm:TS}
    Let the set of arms $\mathcal{A} \subset \mathbb{R}^d$ be finite. Suppose that the reward from playing an arm $a \in \mathcal{A}$ at any iteration, conditioned on the past, is sub-Gaussian with mean\footnote{A random variable $X$ is said to be sub-Gaussian with mean $\beta$ if $\mathbb{E}[e^{t(X-\beta)}] \leq e^{t^2/2}$. } $a^{\top} \mu^\star$. After $T$ iterations, let the guessed best arm $A_T$ be drawn from the empirical distribution of all arms played in the $T$ rounds, i.e., $\mathbb{P}[A_T = a] = \frac{1}{T} \sum_{t=1}^T \mathbf{1}\{a^{(t)} = a \}$ where $a^{(t)}$ denotes the arm played in iteration $t$. Then, 
        \[ \mathbb{P}[A_T \neq a^\star] \leq  \frac{c\log(T)}{T},\]
    where $c \equiv c\left(\mathcal{A}, \mu^\star, \nu \right)$ is a quantity that depends on the problem instance ($\mathcal{A}, \mu^\star$) and algorithm parameter ($\nu$). 
\end{restatable}

Note that the rule for guessing the best arm $A_T$ at the end of the time horizon is slightly different compared to that of Algorithm \ref{alg:linTS}[line \ref{line:reco}]. This result is derived by appealing to a standard reduction between cumulative regret and simple regret for the empirical distribution-based guessing rule \citep{lattimore2020bandit}. This is then combined with a recent logarithmic bound for the cumulative regret for LinTS \citep{tsuchiya2020analysis} on one hand, along with an inequality relating simple regret to the probability of misidentifying the best arm on the other, to obtain the result (the explicit form of $c$ appears in the appendix). We are unaware of any prior result that bounds the identification error probability of linear parametric Thompson sampling, so this result may be of independent interest. 

{\bf Alternative optimization algorithms:} Instead of linear Thompson sampling and gradient descent, one could choose a variety of methods to optimize the unified architecture across the functions $M_1, M_2$ and $\phi$, depending on practical considerations. We have already mentioned the possibility of using gradient-based optimization jointly across all three functions. On the other end, one can employ global optimization methods such as Bayesian optimization \citep{shahriari2015taking} for the continuous space of $\phi$, along with multi-armed bandits for $M_1, M_2$ as we have done here. Of course, even the design of adaptive discrete sampling algorithms for finding the best $M_1,M_2$ is open to a wide variety of possibilities, including best arm identification algorithms for linear bandits \citep{fiez2019sequential}, simulated annealing \citep{rutenbar1989simulated} and evolutionary algorithms \citep{hruschka2009survey}, to name just a few. 

\section{Discussion}
The work introduced by \cite{karjol2023neural} can be considered as a specific instance of our work, when $\rho$ is an identity function, in which the resulting architecture is a composition of an $\localFSymmetric{n^2}$-invariant function and a linear transformation. In this section, we formally analyze the limitations associated with such an approach and establish the non-realizability of $\localFCyclic{k}$-invariant functions using $\localFSymmetric{k}$-invariant functions and a linear transformation for $k \geq 3$. 

\begin{restatable}{theorem}{thmlinearZkSk}
\label{Z_k:S_k linear non-realisability}
Consider the following set of functions, for $k \geq 3$:
\begin{equation*}
    \mathcal{A}_k = \Big \{\phi \circ M \big | M \in \mathbb{R}^{k \times k} \text{(matrix), } \phi \text{ is } \localFSymmetric{k}\text{-invariant} \Big \}.
\end{equation*}
Then, $\exists$ a $\localFCyclic{k}$-invariant function $\psi$ such that $\psi \notin \mathcal{A}_k$.
\end{restatable}

\textit{Proof. (Sketch)} We show the non-realizability of a $\localFCyclic{k}$-invariant function which has a unique value for each orbit. We have, $\Big|\mathcal{O}_{\localFCyclic{k}}(x)\Big| \leq k$. Suppose $\psi = \phi \circ M$, then  $M$ has to be invertible. Then, $\exists$ $\tilde{x}$ such that $\Big|\mathcal{O}_{\localFSymmetric{k}}(M \tilde{x})\Big| = k!$, which leads to a contradiction.

We now conjecture a similar result for $\mathbb{Z}_k$-invariant functions for $n \geq k \geq 3$. 

\begin{conjecture}
Consider the following set of functions, for $n \geq 3$ and $k \leq n$,
\begin{align*}
    \mathcal{A}_n = \Big \{\phi \circ M \big | M \text{ is a linear transformation and } \\ \phi \text{ is } 
    \localFSymmetric{n}-\text{invariant function} \Big \}.
\end{align*}
Then, $\exists$ a $\mathbb{Z}_k$-invariant function $\psi$ such that $\psi \notin \mathcal{A}_n$.
\end{conjecture}

By employing matrix-valued functions as in Theorem \ref{Z_k:S_k}, we gain additional flexibility, allowing us to overcome the above limitations.

\textbf{Canonical form:} 
The proposed architecture utilizes a common $\phi$ i.e., an  $\localFSymmetric{n^2}$-invariant network, while the work proposed in \cite{karjol2023neural} requires $\phi$ be modified depending on the subgroup type. Moreover, our framework yields a canonical form for our overall architecture, as illustrated for the $\localCyclic$ subgroup, given as:


\begin{align*}
        \phi \left(\left[\begin{array}{l}
\left(M_2 \circ \rho \circ M_1\right)(x) \\
\left(I-M_1\right) \left([x \quad  \bf{0}]\right) 
\end{array}\right]\right) \\= \mu \left ( \sum_{i_l \in I} \eta \left(x_{i_l},  x_{\tau \left(i_l \right)} \right) + C, 
Q  \right),
    \label{Canonical Zk}
\end{align*}
where $C = \left( n^2 - k \right) \eta\left(0, 0\right)$ (which is a constant), and $\mu$, $\eta$ denote specific functions and $Q = (I - M_1) [x \quad \bf{0}]$. This follows from the canonical form of $\phi$ as proved in \cite{zaheer2017deep}. Similar results can be obtained for $\localSymmetric$ and $\localDihedral$ subgroups. This allows for a  simple implementation of our architecture for various applications. 

\textbf{Handling non-divisors of $n$:}
We emphasize that the work proposed by \cite{karjol2023neural} for learning $\localCyclic$ (or $\localDihedral$) symmetries is applicable only when $k|n$. In contrast, our framework allows for the discovery of subgroups of type $\localCyclic$ (or $\localDihedral$) for any $| \mathcal{I} \big| = k \leq n$, thus allowing a larger class of subgroups.
\section{Experiments}
\label{Experiments}
We assess the performance of our proposed method in two representative tasks that have been considered in previous related work \cite{kicki2020computationally,zaheer2017deep,karjol2023neural}, one on synthetically generated data (polynomial regression) and the other on a real-world image dataset (image-digit sum) \footnote{While our theoretical results exclude the set $E$ (as defined in the problem statement) from the input domain, we have opted not to do so in our experiments, considering that $E$ is a set with measure zero.}. We discuss additional experiments and potential applications in the appendix section.

\subsection{Polynomial Regression}
In this task, we conduct the model training to learn a $G$-invariant polynomial as studied in \cite{kicki2020computationally}. For example, with $n = 5, k = 4$; $f(x) = x_1x_2x_3x_4 + x_5$ is an $\localFSymmetric{4}$-invariant polynomial function. Note that we also study numerous polynomials of various degrees and give detailed definitions of the polynomials in the supplementary section. To examine the generalization abilities of the proposed method we use only $64$ randomly generated points in $[0, 1]$ for training, whereas use $480$ and $4800$ points for validation and test sets respectively.

\begin{table*}
\small
\parbox{.45\linewidth}{
\begin{tabular}{|l|l|c|}
\hline
Task & $G$  & Accuracy  \\
\hline
\textit{Polynomial Regression} &$\localCyclic$ &100 \\
\textit{Polynomial Regression} &$\localDihedral$  &100 \\
\textit{Image-Digit Sum} &$\localSymmetric$ &100 \\
\hline
\end{tabular}
\renewcommand\thetable{(1.a)}
\caption{Accuracy $(\%)$}
\label{Tab:Accuracy}
}
\parbox{.45\linewidth}{
\begin{tabular}{|l|c|c|c|c|}
\hline
$G$ & $\localCyclic (5)$ & $\localCyclic (7)$ & $\localDihedral (5)$ & $\localDihedral (7)$ \\
\hline
$\localCyclic$ &\textbf{4.2}  &\textbf{6.1}  &8.2  &15.2 \\
$\localDihedral$    &4.7  &7.9  &\textbf{6.3}  &\textbf{10.1} \\
$\localSymmetric$    &11.7  &18.5  &21.3 &34.3 \\
\textit{$M + H$-INV}  &12.3  &-  &23.2  &- \\
\textit{SGD}  &14.4  &17.7  &26.5  &34.4 \\
\hline
\end{tabular}
\renewcommand\thetable{(1.b)} 
\caption{MAE $(\times 10^{-2})$}
\label{Tab:MAE values}
}

\renewcommand\thetable{(1)}
\caption{\textbf{(a)} Estimation accuracy (top 3) for subgroup discovery in polynomial regression and image-digit sum tasks. \textbf{(b)} Mean absolute error ($\times 10^{-2}$) for the regression task with $\localCyclic$ and $\localDihedral$ subgroups. The cardinality ($k = |\mathcal{I}|$) of the index set is given in braces. The first three rows display the top $3$ bandit arm subgroups, with the actual subgroup results highlighted in bold. The $M + H$-INV (only applicable for $k | n$) represents the subgroup discovery method proposed by \cite{karjol2023neural}, which incorporates a composite of linear transformations and an $H$-invariant network. Here, $H \leq \localFSymmetric{n}$ is dependent on the underlying subgroup. The last row represents the proposed architecture entirely trained with SGD.}
\end{table*}

\subsection{Image-Digit Sum}
\vspace{-0.25cm}
The goal of this task is to learn the function representing the sum of  digit labels  of $k$ (out of $n$) images. An input is a set of $n$ images of dimension $28 \times 28$ taken from MNISTm dataset (\cite{loosli2007training}). Using the proposed bandit setting, we discover the underlying subgroup (in this case $\localSymmetric$). Note that, $x_i$ is an image (or $2$D matrix), instead of scalar element.    
\vspace{-0.25cm}
\begin{figure}[htp]
    \centering
    \includegraphics[width=\columnwidth, keepaspectratio, draft=false,  trim={6cm, 3cm, 5cm, 6cm}, clip]{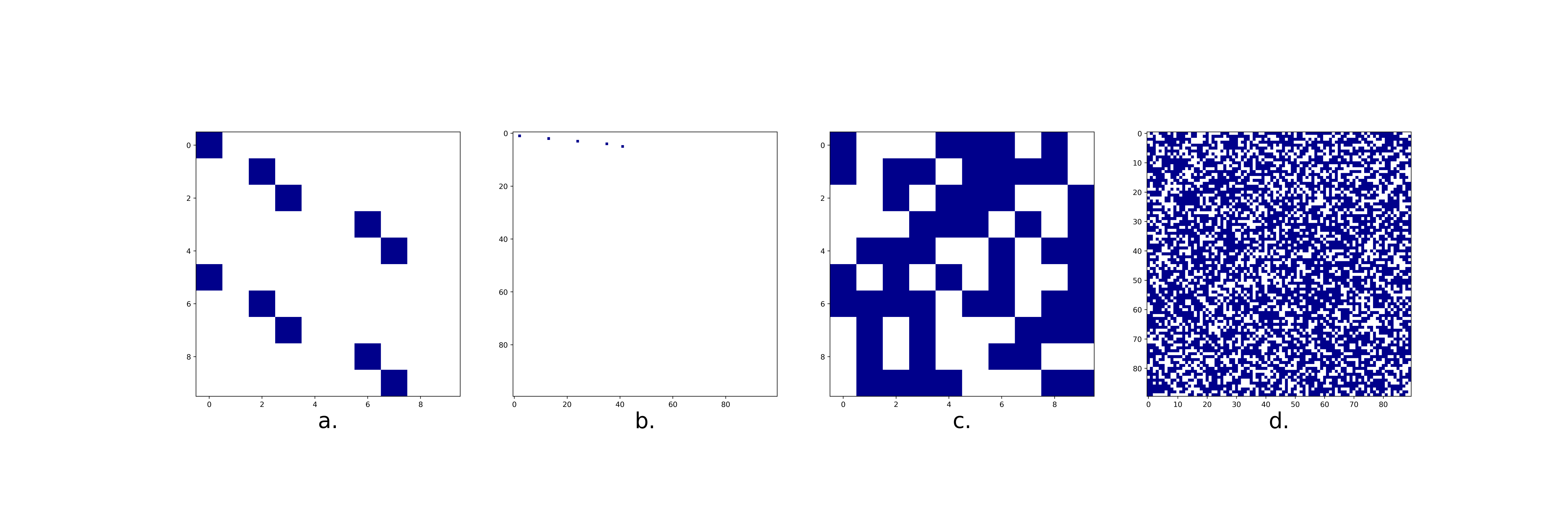}
    \caption{Visualization of the reference (bandit) $M_1$ ({\color{red}{a}}) and $M_2$ ({\color{red}{b}}) matrices, as well as those ({\color{red}{c}}, {\color{red}{d}}) obtained through training our method entirely using SGD for the task of  polynomial regression of $\localCyclic$-invariant function, with $n=10$ and $\mathcal{I} = \{0, 2, 3, 6, 7\}$.}
    \label{fig:visualization}
\end{figure}

\subsection{Results}
\vspace{-0.25cm}
Table \ref{Tab:Accuracy} presents the accuracies achieved in subgroup discovery tasks for image-digit sum ($\localSymmetric$) and polynomial regression ($\localCyclic$ and $\localDihedral$). The reported accuracies correspond to different values of $k$ within the range $[n]$, where $n=10$, and are based on randomly selected index sets $\mathcal{I}$. These accuracies indicate the successful identification of the underlying subgroup within the top 3 bandit arms, as determined by the final $\hat{\mu}$. The training process achieves this outcome within $T = O(n)$ iterations.

In Table \ref{Tab:MAE values}, the top 3 bandit arms corresponds to the best three arms returned by the LinTS algorithm. We note that, in each case the top 3 results is the $S_I, \localCyclic$ or $D_I$ for the correct index set $\mathcal{I}$.

For the polynomial regression task, we also provide the mean absolute error (MAE) values for the top 3 bandit arms obtained. Notably, the MAE corresponding to the actual subgroup is the lowest, indicating successful discovery of the actual subgroup within the top 3. It is worth mentioning that the loss values observed for $\localCyclic$ and $\localDihedral$ subgroups are relatively close, as the only additional group symmetries are the reflections. In addition, we consider the proposed architecture entirely trained with SGD. Our results consistently demonstrate a significant performance improvement over the SGD method across all investigated subgroups in the polynomial regression tasks. Furthermore, we compare our approach with the subgroup discovery method proposed by \cite{karjol2023neural}, which combines linear transformations and an invariant network specifically designed for each subgroup type.


\subsection{Interpretability}
\vspace{-0.25cm}
Bandit sampling inherently yields interpretable outcomes, and an illustrative example $(M_1, M_2)$ of this is demonstrated in Figure \ref{fig:visualization} (a, b). Conversely, training our method solely using SGD results in matrices that lack clear characterization of the underlying subgroup, as depicted in Figure \ref{fig:visualization} (c, d).



\subsection{Limitations and Conclusion}
\vspace{-.25cm}
This work introduces a novel framework for the discovery of discrete symmetry groups. We employ neural architectures trained using a combination of gradient descent and bandit sampling, resulting in interpretable outcomes. Through experiments on both synthetic and real-world datasets, we demonstrate the effectiveness of our approach. It is important to note that this work primarily focuses on theoretical aspects and serves as a proof of concept. In the future, we plan to explore similar approaches for addressing continuous groups and their corresponding applications.



\onecolumn
\vspace{1.5cm}
\aistatstitle{Supplementary Materials}
\vspace{-8cm}
\section{Appendix}
The Appendix Section is organized as follows:-
\begin{itemize}
    \item In Section \ref{Illustration}, we provide an illustration of the proposed method for $G=\localCyclic$ invariance  with $n=4$ and $\mathcal{I} = \{1,2,4\}$. We discuss multi-armed bandits and potential applications in Section \ref{Multi-Armed Bandits} and \ref{Molecular Properties} respectively.
    \item In Section \ref{Additional Experiments}, we discuss additional experiments for convex area estimation and polynomial regression tasks.
    \item In Section \ref{Complete Proofs}, we provide complete proofs for our results with additional theoretical results.
\end{itemize}

\section{Illustration}
\label{Illustration}
\begin{figure}[htp!]
    \centering
    \includegraphics[width= \columnwidth, keepaspectratio, draft=false]{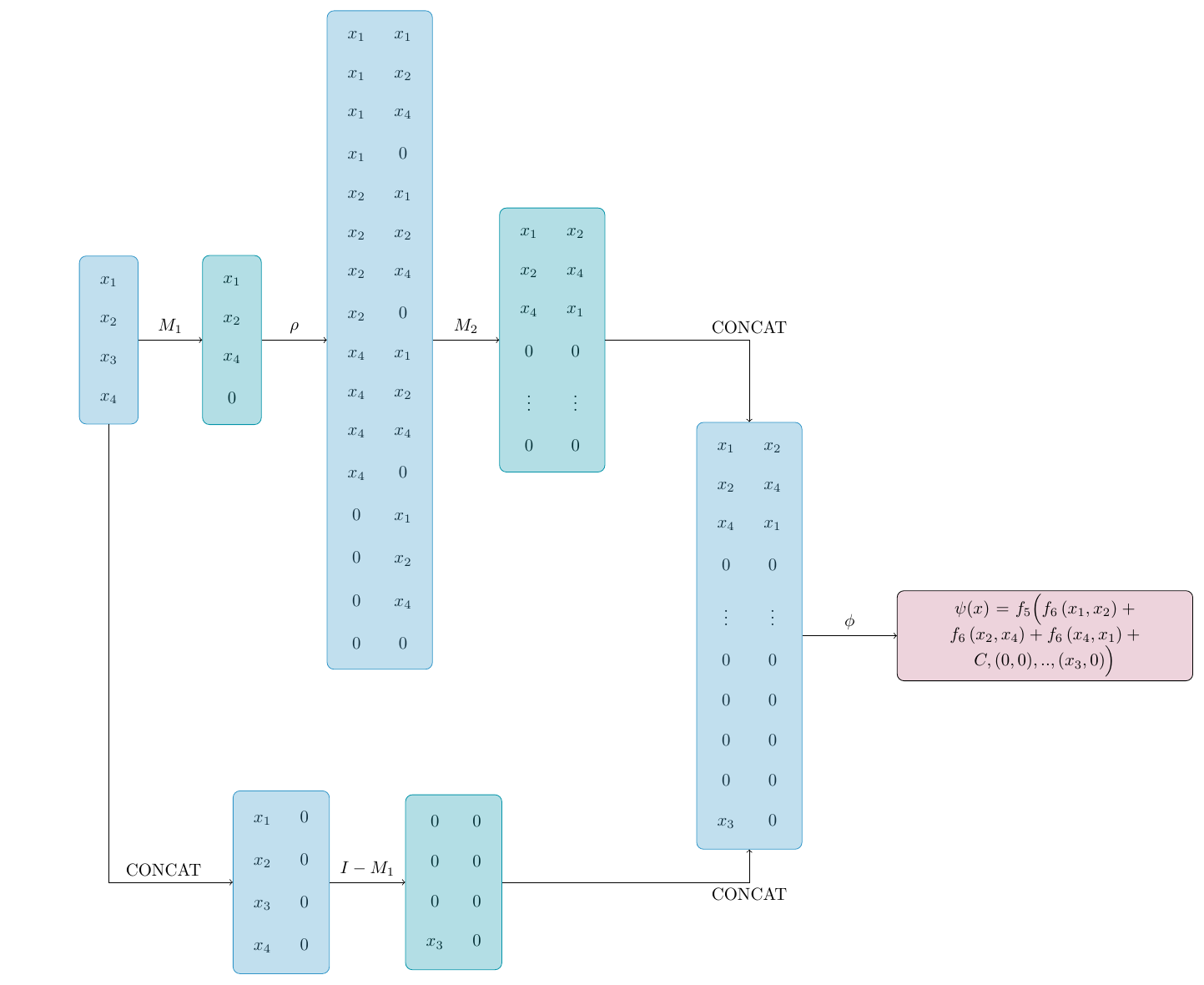}
    \caption{Illustration of the proposed method for $G=\localCyclic$ with $n=4$ and $\mathcal{I} = \{1,2,4\}$. Here $f_5, f_6$ denote some appropriate functions which will be approximated using neural networks.}
    \label{fig:illustration_Z2k}
\end{figure}

\section{Multi-Armed Bandits}
\label{Multi-Armed Bandits}
The Multi-Armed Bandit (MAB) framework is a classical approach for sequential decision-making problems, in which an agent $\mathcal{A}$ selects actions (arms) to minimize the total regret given by $R_T = T \lambda^*- \mathbb{E} \left[\sum_{t=1}^T R_t\right]$  where $\lambda^*$ is the  mean reward  of the optimal arm.

Thompson sampling is a Bayesian approach to the multi-armed bandit problem. It works by sampling from a posterior distribution over the expected rewards of each arm, and then selecting the arm with the highest sampled reward. The posterior distribution is updated after each round of play, based on the observed rewards. In this setting, each arm (action) is associated with a context or feature vector \(x\), and the goal is to learn a linear model that predicts the expected reward for each arm given its context. Let $X_t$ be the context vector at time $t$, $A_t$ be the chosen arm at time $t$, and $R_t$ be the observed reward at time \(t\). The algorithm assumes a prior distribution over the model parameters $\mu$ (e.g., multivariate Gaussian distribution). At each iteration, Thompson Sampling samples a parameter vector \(\mu\) from the posterior distribution. Then, it estimates the expected reward for each arm by computing the inner product between the sampled \(\mu\) and the corresponding context vector \(x\). The arm with the highest estimated reward is chosen and pulled. After observing the reward, the posterior distribution is updated using Bayesian inference to obtain a new posterior distribution, taking into account the new data. This update process is typically performed using conjugate priors or approximate methods like Markov Chain Monte Carlo (MCMC) or variational inference. The algorithm continues to update the posterior distribution and select arms based on the sampled parameters, enabling it to learn the optimal policy in a contextual bandit setting.

Thompson Sampling has been proven to be asymptotically optimal, meaning that as $T \rightarrow \infty$, the regret of the algorithm is bounded by a logarithmic function of \(T\). Formally, it has been shown that $\lim \limits_{T \to \infty} \frac{R_T}{T} = 0 $, where $R_T$ represents the regret after $T$ rounds. This result guarantees that over time, Thompson Sampling converges to the optimal arm and achieves maximum total reward. The logarithmic regret bound demonstrates the efficiency of the algorithm in balancing exploration and exploitation, leading to near-optimal performance in the long run.

\section{Additional Experiments}
\label{Additional Experiments}
\begin{table}[htp!]
\caption{Estimation Accuracy $(\%)$}
\vspace{0.1em}
\centering
\begin{tabular}{|l|l|c|}
\hline
Task & $G$  & Accuracy  \\
\hline
\textit{Convex Area} &$\localDihedral$ &100 \\
\textit{$\localSymmetric$ (4)} &$\localSymmetric$  &100 \\
\hline
\end{tabular}
\renewcommand\thetable{(1.a)}
\label{Tab:New Accuracy}
\end{table}
Table \ref{Tab:New Accuracy} presents the accuracies (top 3) achieved in subgroup discovery tasks on two tasks: (i) convex quadrangle area estimation. (ii) $\localSymmetric$-invariant polynomial regression. The cardinality ($k = |\mathcal{I}|$) of the index set is given in braces. 

\textit {Convex area estimation}. In this task, we estimate the area of convex quadrilaterals which are invariant to cyclic shifts and reflections of the input coordinates, i.e., a $\localDihedral$-invariant function ($|\mathcal{I}| = 4$). The input is the $(x, y)$ coordinates of the four points of the quadrilateral lying in $\mathbb{R}^{4 \times 2}$. The training data consists of 256 examples (randomly generated convex quadrangles with their areas), while the validation dataset contains 1024 examples. Note that, the coordinates are randomly sampled from $[0, 2]$ and the area takes value in $(0, 1]$ respectively.

\textit{Polynomial regression}. Here, we consider $\localSymmetric$-invariant polynomial regression task. The training dataset consists of 64 randomly generated data points in $[0,1]$, whereas 480 points were used for the validation set.

For all our experiments, we observe the subgroup discovery in $O(n)$ iterations. At each iteration, we run the model for $400$ epochs ($3$ for image-digit sum) with batch size of $16$ and decaying learning rate schedule on \textit{NVIDIA A6000 GPU's}. We report the accuracy obtained across $5$ trails with different index set $I$.

\begin{table}[htp!]
\caption{Definition of Polynomials} 
\label{tab:Poly}
\begin{center}
\resizebox{0.7\columnwidth}{!}{\begin{tabular}{|l| l|}
\hline 
\textbf{INVARIANCE}  &\textbf{POLYNOMIAL}  \\
\hline
$\localSymmetric$ (4)    &$x_{1}x_{2}x_{3}x_{4} + x_{5}$\\
$\localCyclic$ (5)    &$x_1x_2^2 + x_2x_3^2 + x_3x_6^2 + x_6x_7^2 + x_7x_1^2$		  \\
$\localCyclic$ (7)    &$x_1x_2^2 + x_2x_3^2 + x_3x_6^2 + x_6x_7^2 + x_7x_9^2 + x_9x_{10}^2 + x_{10}x_1^2$		  \\
$\localDihedral$ (5)    &$x_1x_2^2 + x_2x_3^2 + x_3x_6^2 + x_6x_7^2 +  x_7x_1^2 +  x_1x_7^2 +  x_7x_6^2 + x_6x_3^2  + x_3x_2^2  + x_2x_1^2$		\\
$\localDihedral$ (7)    &$x_1x_2^2 + x_2x_3^2 + x_3x_6^2 + x_6x_7^2 + x_7x_9^2 + x_9x_{10}^2 + x_{10}x_1^2 + x_{1}x_{10}^2 + ... + x_2x_1^2$		\\
\hline
\end{tabular}}
\end{center}
\end{table}
Table (6): The exact definitions of the polynomials used in experiments is given in Table \ref{tab:Poly}. For $\localCyclic$ and $\localDihedral$ the input is a vector in $[0, 1]^{10}$ given as; $x = [x_{1}, x_{2}, ..., x_{10}]$ whereas for $\localSymmetric$ it is a vector in $[0, 1]^{5}$ given as; $x = [x_{1}, x_{2}, ..., x_{5}]$. In this example, the index set $\mathcal{I}$ is chosen to be $[1, 2, 3, 4]$, $[1, 2, 3, 6, 7]$, and $[1, 2, 3, 6, 7, 9, 10]$ respectively.


\section{Potential applications: Molecular Properties}
\label{Molecular Properties}
In our research, we introduce a novel framework for the discovery of discrete invariance in functions, particularly concerning their behavior under a set of discrete symmetries. One compelling application of this framework emerges in the domain of molecular properties and their underlying symmetries. Consider a scenario where a collection of molecules exhibits a shared property, and it is hypothesized that this property is rooted in the presence of a common point group or discrete symmetry group (\cite{carter1997molecular}). Our framework can serve as a powerful tool to discover this common point group and explore this hypothesis.

However, it is essential to note that the successful application of our framework necessitates the proper representation of molecules in terms of graphs or other suitable data structures. Additionally, we advocate the construction of backbone-invariant neural networks, such as $\phi$, tailored to these data structures, specifically designed to withstand certain symmetry transformations (known as symmetry elements). This prerequisite forms a distinct yet intriguing avenue of research, wherein our framework for symmetry discovery plays a pivotal role. By leveraging our method, researchers can effectively tackle the challenging task of identifying and understanding the discrete symmetries that underlie molecular properties, promising significant advancements in the fields of chemistry, materials science, and drug discovery.

\section{Complete Proofs and Additional Theoretical Results}
\label{Complete Proofs}
\begin{proposition}[Cayley's Theorem]
\label{prop:cayley} 
Let $G$ be a group, and let $H$ be a subgroup. Let $G/H$ be the set of left cosets of $H$ in $G$. Let $N$ be the normal core of $H$ in $G$, defined to be the intersection of the conjugates of $H$ in $G$. Then the quotient group  $G/N$ is isomorphic to a subgroup of $Sym(G/H)$. More specifically, it states that every group $G$ is isomorphic to a subgroup of the symmetric group.
\end{proposition}

\subsection{Proof of Theorem \ref{Z_k:S_k}}
\thmZkSk*
\begin{proof}
    \textbf{Step 1}: First, we show that the $\rho : X \rightarrow \mathbb{R}^k$ is an injective function, where $X = [0,1]^k$. Suppose $\rho(x) = \rho(y)$, for some $x = [x_1, x_2, \dots x_k]^T$ and $y = [y_1, y_2, \dots y_k]^T$. Then,
    \begin{equation}        
        \begin{bmatrix}
            x_1, &x_2 \\
            x_2 &x_3 \\
            \vdots &\vdots \\
            x_k &x_1
        \end{bmatrix}
        = 
        \begin{bmatrix}
        y_1 &y_2 \\
        y_2 &y_3 \\
        \vdots &\vdots \\ 
        y_k &y_1
        \end{bmatrix},
    \end{equation}
    thus,
    \begin{equation}
        (x_1, x_2) = (y_1, y_2), \;
        (x_2, x_3) = (y_2, y_3), \;
        \dots, \;
        (x_{k-1}, x_k) = (y_{k-1}, y_k), \;
        (x_k, x_1) = (y_k, y_1). 
    \end{equation}

    Thus, we get, $x_i = y_i, \;  \forall i \in [k]$. Hence, $\rho$ is injective.

    In addition, $\rho^{-1} : \rho(X) \rightarrow X$ is given by 
    \begin{equation}
        \rho^{-1}\left(        \begin{bmatrix}
            x_1, &x_2 \\
            x_2 &x_3 \\
            \vdots &\vdots \\
            x_k &x_1
        \end{bmatrix} \right) =         \begin{bmatrix}
            x_1 \\
            x_2 \\
            \vdots  \\
            x_k
        \end{bmatrix}.
        \label{eq:rho_inv}
    \end{equation}
    \textbf{Step 2}: It is obvious to see that $\rho$ is a $\localFCyclic{k}$-equivariant function, i.e.,
    \begin{equation}
        \rho(h \cdot x) = h \cdot \rho(x), \quad \forall h \in \localFCyclic{k}
    \end{equation}
    
    \textbf{Step 3}:
    We now show that, for any $g \in \localFSymmetric{k}$, $g \cdot \rho(x) \in \text{ Im}(\rho)$ if and only if $g \cdot \rho(x) = h \cdot \rho(x)$ for some $h \in \localFCyclic{k}$. In other words, any permutation (row wise) of  $\rho(x)$ correspond to some cyclic shift of $\rho(x)$.

    From Step $2$, we get that, if $g \in \localFCyclic{k}$, then $g \cdot \rho(x) = \rho(g \cdot x) $. Thus, $g \cdot \rho(x) \in Im(\rho)$.

    Suppose $g \cdot \rho(x) \in Im(\rho)$ for some $g \in \localFSymmetric{k}$.
    Since $ \rho(x) \in Im(\rho)$, we have
    \begin{align}
        \rho(x) &=         
        \begin{bmatrix}
            x_1, &x_2 \\
            x_2 &x_3 \\
            \vdots &\vdots \\
            x_k &x_1
        \end{bmatrix} \nonumber \\
        g \cdot \rho(x) &=  
        \begin{bmatrix}
        x_{g(1)} &x_{\tau(g(1))}) \\
        x_{g(2)} &x_{\tau(g(2))}) \\
        \vdots &\vdots  \\
        x_{g(k)} &x_{\tau(g(k))}          
        \end{bmatrix} \label{eq:e1} \\
        \rho^{-1} (g \cdot \rho(x)) &= 
        \begin{bmatrix}
            x_{g(1)} \\
            x_{g(2)} \\
            \vdots  \\
            x_{g(k)}
        \end{bmatrix} \quad \left( g \cdot \rho(x) \in Im(\rho) \text{ and applying \eqref{eq:rho_inv}} \right)
        \nonumber \\
        \rho(\rho^{-1} (g \cdot \rho(x))) = g \cdot \rho(x) &= 
        \begin{bmatrix}
            x_{g(1)} & x_{g(2)} \\
            x_{g(2)} &x_{g(3)} \\
            \vdots &\vdots \\
            x_{g(k)} &x_{g(1)}
        \end{bmatrix} 
        \label{eq:e2}
    \end{align}
    where  $\tau$  is cyclic shift operator defined as $\tau(j) = (j \mod k) + 1$.
    
From eq. \eqref{eq:e1} and \eqref{eq:e2}, (substituting $w = g(1)$), we get,
\begin{equation}
        g \cdot \rho(x) 
        =
        g \cdot \begin{bmatrix}
        x_1 &x_2 \\
        x_2 &x_3 \\
        \vdots &\vdots \\
        x_k &x_1
        \end{bmatrix}
        =
        \begin{bmatrix}
        x_{w} &x_{\tau(w)} \\
        x_{\tau(w)} &x_{\tau^2(w)} \\
        \vdots &\vdots \\
        x_{\tau^{k-1}(w)} &x_{\tau^k(w)}
    \end{bmatrix},
\end{equation}
which is nothing but cyclic shift of $\rho(x)$. Thus,  $g \cdot \rho(x) = h \cdot \rho(x)$ for some $h \in \localFCyclic{k}$.

    \textbf{Step 4}: Claim: The following map is injective:
    \begin{equation}
        \mathcal{O}_{\localFCyclic{k}}(x) \mapsto \mathcal{O}_{\localFSymmetric{k}}\left( \rho(x)\right)
    \end{equation}
    First we will show that, this map is well-defined. Suppose, $y \in \mathcal{O}_{\localFCyclic{k}}(x)$, then $\mathcal{O}_{\localFCyclic{k}}(y) = \mathcal{O}_{\localFCyclic{k}}(x)$ and $y = h \cdot x$ for some $h \in \localFCyclic{k}$.
    \begin{align}
        \implies \mathcal{O}_{\localFSymmetric{k}}\left( \rho(y)\right) &=  \mathcal{O}_{\localFSymmetric{k}}\left( \rho(h \cdot x)\right) & \;\nonumber \\
        & = \mathcal{O}_{\localFSymmetric{k}}\left( h \cdot \rho(x)\right) \quad &\text{(from step 2)} \nonumber \\
        & = \mathcal{O}_{\localFSymmetric{k}}\left( \rho(x)\right) \quad &\text{(from the definition of orbit).}
    \end{align}
Hence, the map is well-defined. 

Suppose, $\mathcal{O}_{\localFSymmetric{k}}\left( \rho(x)\right) = \mathcal{O}_{\localFSymmetric{k}}\left( \rho(y)\right)$ for some $x, y \in [0,1]^k$, then
\begin{align}
    \rho(y) &\in \mathcal{O}_{\localFSymmetric{k}}\left( \rho(x)\right) \quad &\text{(from the definition of orbit)} \nonumber \\
    \rho(y) &= g \cdot \rho(x) \quad &(\text{for some } g \in \localFSymmetric{k}) \nonumber \\
    g \cdot \rho(x) &\in \; Im \left(\rho \right) & \; \nonumber \\
    g &\in \localFCyclic{k} \quad &(\text{from step 3}) \nonumber \\
    \rho(y) &= g \cdot \rho(x) = \rho \left(g \cdot x \right) \quad &(\text{from step 2}) \nonumber \\
    y &= g \cdot x \quad &(\text{from step 1}) \nonumber \\
    y &\in \mathcal{O}_{\localFCyclic{k}}(x) \nonumber & \; \\
    \mathcal{O}_{\localFCyclic{k}}(y) &= \mathcal{O}_{\localFCyclic{k}}(x). &
\end{align}
    This implies that  each $\mathcal{O}_{\localFCyclic{k}}(x)$ orbit is uniquely mapped to $\mathcal{O}_{\localFSymmetric{k}}(\rho(x))$.  From this, it follows that by defining the $\localFSymmetric{k}$-invariant function $\phi$ to take the same value across any orbit of the form $\mathcal{O}_{\localFSymmetric{k}}(\rho(x))$ as $\psi$ does across the orbit $\mathcal{O}_{\localFCyclic{k}}(x)$ (and an arbitrary value across orbits not of the form $\mathcal{O}_{\localFSymmetric{k}}(\rho(x))$), we obtain the result. 
\end{proof}

\subsection{Additional Theoretical Results}
\begin{theorem}
    \label{D2k_S2k}
    Let $\psi: X \rightarrow \mathbb{R} $ be $\localFDihedral{k}$-invariant. There exists an $\localFSymmetric{2k}$-invariant function $\phi : [0,1]^{2k \times 2} \rightarrow \mathbb{R}$ and $\rho: X \rightarrow [0,1]^{2k \times 2}$,  such that 
    \begin{equation}
        \label{thm:d2k}
        \psi = \phi \circ \rho,
    \end{equation}
    where $\rho$ is defined as, 
    \begin{align}
        \begin{bmatrix}
        x_1 \\
        \vdots \\
        x_k
        \end{bmatrix} 
        &\mapsto
        \begin{bmatrix}
        x_1 &x_2 \\
        x_2 &x_1 \\
        x_2 &x_3 \\
        x_3 &x_2 \\
        \vdots &\vdots\\
        x_k &x_1 \\
        x_1 &x_k
        \end{bmatrix}
        \label{rho def D_2k}
    \end{align}
    
\end{theorem}
\begin{proof}
As discussed in Theorem \ref{Z_k:S_k}, the goal is to map each of the $\localFDihedral{k}$-orbit in the input domain $X$ uniquely to a $\localFSymmetric{2k}$-orbit in $\rho(X)$. 

\textbf{Step 1:} First, we show that $\rho$ is injective.
Suppose $\rho(x) = \rho(y)$ for some 
\begin{equation}
x = \begin{bmatrix}
    x_1 \\
    x_2 \\
    \vdots \\
    x_k
\end{bmatrix}, 
y = \begin{bmatrix}
    y_1 \\
    y_2 \\
    \vdots \\
    y_k
\end{bmatrix}
.
\end{equation}

\text{Then}, 
\begin{equation}
\begin{bmatrix}
        x_1 &x_2 \\
        x_2 &x_1 \\
        x_2 &x_3 \\
        x_3 &x_2 \\
        \vdots &\vdots\\
        x_k &x_1 \\
        x_1 &x_k
\end{bmatrix} 
=
\begin{bmatrix}
        y_1 &y_2 \\
        y_2 &y_1 \\
        y_2 &y_3 \\
        y_3 &y_2 \\
        \vdots &\vdots\\
        y_k &y_1 \\
        y_1 &y_k
\end{bmatrix}.
\end{equation}
Hence,
$x_1 = y_1, \; x_2 = y_2, \; \dots , x_k = y_k, \;$ Therefore, $x = y$, and thus, $\rho$ is injective.

\textbf{Step 2:} $\rho$ is equivariant function, i.e., for any $h \in \localFDihedral{k}$, we have $\rho(h \cdot x) = g \cdot \rho(x)$ for some $g \in \localFSymmetric{2k}$.

\textbf{Step 3:} Suppose $g \cdot \rho(x) \in \operatorname{Im}(\rho)$ for some $g \in S_k$, then $g \cdot \rho(x)=\rho(h \cdot x)$ for some $h \in D_{2 k}$. 

\underline{Case 1}: If $(g \cdot \rho(x))[1]=\rho(x)[2 u-1] \hspace{0.3em} \text{for} \hspace{0.3em} u \in[k]$, then using the definition of $\rho(x)$ and since $g \cdot \rho(x) \in \operatorname{Im}(\rho)$, we get that, 
\begin{equation*}
g \cdot \rho(x)[3,1] = g \cdot \rho(x)[1,2].
\end{equation*}
    
Thus,
\begin{equation*}
g \cdot \rho(x) = \begin{bmatrix} 
    x_u &x_{\tau(u)}    \\
    x_{\tau(u)} &x_u \\
    x_{\tau(u)} &* \\
    * &x_{\tau(u)} \\
    \quad \quad \vdots &  \\
    * &x_u \\
    x_u &*
    \end{bmatrix},
\end{equation*}
where the `*' symbols represent values that we will discover next.

The uniqueness of $x_i$'s (i.e., we exclude the set $E$ from the input domain so that each of the $x_i$'s are unique)  leads to 
    $g \cdot \rho(x)[3] = [x_{\tau(u)} \quad \quad x_{\tau^2(u)}]$  (since, $ g \cdot \rho(x) [2] = [x_{\tau(u)} \quad \quad x_u] $). Thus, $g \cdot \rho(x)[4] = [ x_{\tau^2(u)} \quad \quad x_{\tau(u)}]$ and, 
    \begin{equation*}
    g \cdot \rho(x) = \begin{bmatrix} 
    x_u &x_{\tau(u)}    \\
    x_{\tau(u)} &x_u \\
    x_{\tau(u)} &x_{\tau^2(u)} \\
    x_{\tau^2(u)} &x_{\tau(u)} \\
    \quad \quad \vdots &  \\
    * &x_u \\
    x_u &* 
    \end{bmatrix}. 
    \end{equation*}

Continuing this process, we get,
\begin{equation*}
g \cdot \rho(x) = \begin{bmatrix} 
    x_u &x_{\tau(u)}    \\
    x_{\tau(u)} &x_u \\
    x_{\tau(u)} &x_{\tau^2(u)} \\
    x_{\tau^2(u)} &x_{\tau(u)} \\
    \quad \quad \vdots &  \\
    x_{\tau^{k-1}(u)} &x_u \\
    x_u &x_{\tau^{k-1}(u)} 
    \end{bmatrix}.
\end{equation*}

Thus, $g \cdot \rho(x) = \rho([x_u, x_{\tau(u)}, x_{\tau^2(u)}, \dots x_{\tau^{k-1}(u)}]^T) = \rho(h \cdot x)$ for some $h \in \localFCyclic{k}$, then, $ h \in \localFDihedral{k}$.
$$$$
\underline{Case 2}: If $(g \cdot \rho(x))[1]=\rho(x)[2 u]$ with $ u \in[k]$, then we get,

\begin{equation*}
g \cdot \rho(x) = \begin{bmatrix} 
    x_{\tau(u)} &x_u     \\
    x_u &x_{\tau(u)}     \\
    x_u &x_{\tau^{k-1}(u)}  \\
    x_{\tau^{k-1}(u)} &x_u \\
    \quad \quad \vdots &  \\
    x_{\tau^2(u)} &x_{\tau(u)} \\
    x_{\tau(u)} &x_{\tau^2(u)} 
    \end{bmatrix}.
\end{equation*}

Thus, we obtain that: 
\begin{itemize}
    \item $g \cdot \rho(x) = \rho([x_{\tau(u)}, x_u, x_{\tau^{k-1}(u)}, \dots x_{\tau^2(u)}]^T)$.
    \item $g \cdot \rho(x) = \rho(\Tilde{h} \cdot x) $ where $\Tilde{h} \in \localFDihedral{k} \setminus \localFCyclic{k}$ (i.e., reflection around the center followed by a cyclic shift).
\end{itemize}

To summarize, we now have the following:
\begin{itemize}
    \item For any $ h \in \localFDihedral{k}$,  $\rho(h \cdot x) = g \cdot \rho(x)$
    for some $g \in \localFSymmetric{2k}$ (from step 2).
    \item For any $g \in \localFSymmetric{2k}$, such that $g \cdot \rho(x) \in Im(\rho)$ (i.e., $g \cdot \rho(x) = \rho(y)$ for some $y \in X$),
    \begin{equation*}
        g \cdot \rho(x) = \rho(h \cdot x), \hspace{0.2em} \text{for some $h \in \localFDihedral{k}$}, \quad (\text{from step 3})
    \end{equation*}
\end{itemize}

Using this, we can show the mapping of orbits as discussed in Step 4 of the proof in Theorem \ref{Z_k:S_k}.
\end{proof}

\subsection{Exclusion of the Set $E$}

As stated in the problem statement, the input domain is defined as $X = [0, 1]^n \setminus E$, representing the input (instance) domain. Here, $E = \{[x_1, x_2, \ldots, x_n]^T \in [0, 1]^n : x_i = x_j \text{ for some } i, j \in [n] \text{ with } i \neq j \}$. The exclusion of this set is necessary for cases involving $\localDihedral$ invariance, which is an integral part of the overall framework. It should be noted that this exclusion is not required for $\localCyclic$-invariance or $\localSymmetric$-invariance.

The significance of excluding $E$ in the context of $\localDihedral$-invariance is illustrated by the following example:

\begin{equation}
    x = 
    \begin{bmatrix}
        1 \\
        2 \\
        3 \\
        1 \\
        4 \\
        5 \\
        6 \\
        7 \\
        8 \\
        9 \\
        10 \\
        11
    \end{bmatrix}
    \xmapsto[]{\rho} 
    \rho(x) = \begin{bmatrix}
        1 & 2\\
        2 & 1\\
        2 & 3\\
        3 & 2\\
        3 & 1\\
        1 & 3\\
        1 & 4\\
        4 & 1\\
        4 & 5\\
        5 & 4\\
        5 & 6\\
        6 & 5\\
        6 & 7\\
        7 & 6\\
        7 & 8\\
        8 & 7\\
        8 & 9\\
        9 & 8\\
        9 & 10\\
        10 & 9\\
        10 & 11\\
        11 & 10\\
        11 & 1\\
        1 & 11
    \end{bmatrix}
    \xmapsto[]{g}
    g \cdot \rho(x) =
    \begin{bmatrix}
        2 & 1\\
        1 & 2\\
        1 & 4\\
        4 & 1\\
        4 & 5\\
        5 & 4\\
        5 & 6\\
        6 & 5\\
        6 & 7\\
        7 & 6\\
        7 & 8\\
        8 & 7\\
        8 & 9\\
        9 & 8\\
        9 & 10\\
        10 & 9\\
        10 & 11\\
        11 & 10\\
        11 & 1\\
        1 & 11\\
        1 & 3\\
        3 & 1\\
        3 & 2\\
        2 & 3
    \end{bmatrix} 
    \xmapsto[]{\rho^{-1}}
    \begin{bmatrix}
        2 \\
        1 \\
        4 \\
        5 \\
        6 \\
        7 \\
        8 \\
        9 \\
        10 \\
        11 \\
        1 \\
        3
    \end{bmatrix} = y,
\end{equation}

where $g \in S_{24}$. Here, $y = \Tilde{h} \cdot x$ for some permutation $\Tilde{h}$, but $\Tilde{h} \notin D_{24}$. It is important to note that the elements $x_i$ are not unique (in this example, the value '1' is repeated twice), indicating that $x \in E$.

\subsection{Proof of Theorem \ref{Main framework}}
\textit{Proof}. 
We will prove the result for $\localCyclic$-invariant function (part (b)). Similar steps hold for other variants. As stated in Theorem. \ref{Z_k:S_k}, any $\localFCyclic{k}$-invariant function $\psi$ can be written as a composition of an $\localFSymmetric{k}$-invariant function and a specific non-linear function which is defined in \eqref{rho def Z_k}. If we apply canonical form for $\localFSymmetric{k}$-invariant function as given by \cite{zaheer2017deep}, we get,
\begin{equation}
    \psi(x) = f_1 \left( \sum\limits_{i \in  [k]} f_2 \left(x_i, x_{\tau(i)}\right) \right),
    \label{Canonical Zk v1}
\end{equation}
for some functions $f_1$ and $f_2$.

Similarly any $\localCyclic$-invariant function $\psi$ can be written as (\cite{karjol2023neural}),
\begin{equation}
    \psi(x) = f_3 \left( \sum\limits_{i \in  \mathcal{I}} f_4 \left(x_i,  x_{\tau(i)}\right), Q \right),
    \label{eq:Canonical Form comp 1}
\end{equation}
for some functions $f_3$ and $f_4$, where $Q = (I - M_1) [x \quad \bf{0}]$.

Thus, the goal is show that, the function  
$$
x \mapsto \phi \left(\left[\begin{array}{l}
\left(M_2 \circ \rho \circ M_1\right)(x) \\
\left(I-M_1\right) \left([x \quad  \bf{0}]\right) 
\end{array}\right]\right)
$$
has an equivalent form, for appropriately chosen $M_1$ and $M_2$. With  $M_1$ chosen as in \eqref{common M1}, we get,
\begin{equation}
    \left(M_1 x \right)[i]=\left\{\begin{array}{l}
x_i \\
0
\end{array} \quad \text { if } i \in I\right.
\end{equation}

Then applying the function $\rho$, we get that
$\left\{(x_i, x_j) \mid i, j \in \mathcal{I}, \; i \neq j \right\}$ will be the set of non-zero elements of the vector $\left(\rho \circ M_1\right)(x)$.

If we choose $M_2$ as stated in \eqref{eq:M2_Zk} for $\localCyclic$-invariant function, we obtain that
$\left\{(x_i, x_{\tau(i)}) \mid i \in \mathcal{I} \right\}$ will be the set of non-zero elements of the vector $\left( M_2 \circ \rho \circ M_1\right)(x)$. Then, applying canonical form for $\localFSymmetric{n^2}$-invariant function as given by \cite{zaheer2017deep}, we get,
\begin{equation}
    \phi \left(\left[\begin{array}{l}
\left(M_2 \circ \rho \circ M_1\right)(x) \\
\left(I-M_1\right) \left([x \quad  \bf{0}]\right) 
\end{array}\right]\right) =  f_5\left(\sum_{i \in \mathcal{I}} f_6\left(x_i, x_{\tau(i)}\right) + L f_4(0, 0), Q\right), 
    \label{eq:Canonical Form comp 2}
\end{equation}
where  $L$ is constant and $f_5$ and $f_6$ are some functions.  We observe that \eqref{eq:Canonical Form comp 1} and \eqref{eq:Canonical Form comp 2} have an equivalent form up to a bias term, which can subsumed in $f_3$ and $f_4$. Thus, we conclude that any $\localCyclic$-invariant function can be represented as a function  of the form, $
x \mapsto \phi \left(\left[\begin{array}{l}
\left(M_2 \circ \rho \circ M_1\right)(x) \\
\left(I-M_1\right) \left([x \quad  \bf{0}]\right) 
\end{array}\right]\right).
$

\subsection{Proof of Lemma \ref{diffeomorphism}}
\begin{proof}
To prove the claim, we need to endow $Y = \rho(X)$ with a topology. First, we observe that, for any $y = \begin{bmatrix}
     y_1 &y_2 \\
     y_2 &y_3 \\
    \vdots &\vdots\\
     y_k  &y_1 
\end{bmatrix}$, can be written as a vector of the form $\left[y_1, y_2, y_2, y_3, y_3, \dots y_k, y_k, y_1 \right]^T \in \mathbb{R}^{2k}$. Thus we can employ subspace topology of the standard topology of $\mathbb{R}^{2k}$. It is obvious to see that $\rho$ is bijective with $\rho^{-1}$ defined as:

\begin{equation*}
    \begin{bmatrix}
     y_1 &y_2 \\
     y_2 &y_3 \\
    \vdots &\vdots\\
     y_k  &y_1 
\end{bmatrix} \mapsto \begin{bmatrix}
     y_1  \\
     y_2  \\
    \vdots \\
     y_k   
\end{bmatrix}
\end{equation*}
Thus, since $\rho$ and $\rho^{-1}$ are smooth with respect to the subspace topology, $\rho$ is a diffeomorphism.
\end{proof}

\subsection{Proof of Theorem \ref{Product group:S_l}}
\thmProductGroups*
\begin{proof}

Upon an analysis of different components of $\rho(x)$ corresponding to various component groups, it becomes evident that $\rho$ is both injective and equivariant. Next, we  need to establish the orbit mapping, similar to the proofs provided in Theorem \ref{Z_k:S_k} and Theorem \ref{D2k_S2k}.

Since, $\rho$ is  equivariant, it is sufficient to prove that,  for any $g \in \localFSymmetric{l}$ such that $g \cdot \rho(x) \in \operatorname{Im}(\rho)$ (i.e., $g \cdot \rho(x) = \rho(y)$ for some $y \in X$), we have:

\begin{equation}
    g \cdot \rho(x) = \rho(h \cdot x) \label{eq:equivalent_claim_product_groups}
\end{equation}

for some $h \in G$.

Now, we proceed to show that permutations occur solely within the component groups. To do this, let's assume $g \cdot \rho(x) \in \operatorname{Im}(\rho)$. Then, we can express it as:

\begin{equation}
    g \cdot \rho(x) = \left(g_1 \cdot  \left(\rho(x)[1:k_1]\right), \quad g_2 \cdot \left(\rho(x)[k_1+1:k_2]\right) \quad \dots \quad g_L \cdot \left(\rho(x)[u:l]\right) \right) \label{eq:component_wise}
\end{equation}

Here, $\rho(x)[i_1:i_2]$ represents a portion of the vector $\rho(x)$ corresponding to a component group $G_i$. We'll now analyze the effects of the permutations on elements associated with different component subgroups $G_i$.

Without loss of generality, let $G_1 = D_{\mathcal{I}_1}$, where $|\mathcal{I}_1|$ is the largest cardinality among component groups of type $\localDihedral$.

Consider the first element of $g \cdot \rho(x).$

Suppose $g \cdot \rho(x)[1] = \rho(x)[u] = [x_i \quad x_j]$ for some $x_i$ and $x_j$:

\begin{align}
    \text{Suppose, } g \cdot \rho(x)[1] &= [x_i \quad x_j] \nonumber \\
    \implies g \cdot \rho(x)[2] &= [x_j \quad x_i] \nonumber
\end{align}

This implies that $\rho(x)[u]$ corresponds to some dihedral group $D_{\mathcal{I}'}$. Continuing the analysis as done in step 3 of the proof of Theorem \ref{D2k_S2k}, we arrive at:

\begin{equation}
    \begin{bmatrix}
        x_i & x_j \\
        x_j & x_i \\
        x_j & x_l \\
        x_l & x_j \\
        \vdots & \vdots \\
        x_p & x_i \\
        x_i & x_l
    \end{bmatrix}
\end{equation}

We now have $2|\mathcal{I}_1|$ elements corresponding to a dihedral group. However, $|\mathcal{I}_1|$ is the largest cardinality among dihedral groups, and no two component subgroups are isomorphic. Hence, we conclude:

\begin{equation}
    \mathcal{I}_1 = \mathcal{I}' \nonumber
\end{equation}

Consequently, the permutations occur within the dihedral component, and we have:

\begin{equation}
    \left(g \cdot \rho(x) \right)[1:k_1] = g_1 \cdot \left( \rho(x)[1:k_1] \right), \quad \text{for some } g_1 \in G_1 \nonumber
\end{equation}

Next, we consider the second-largest dihedral component and continue the analysis. Similarly, we can apply the same reasoning for groups of the type $\localCyclic$ and $\localSymmetric$. This confirms the assertion presented in equation \eqref{eq:component_wise}. Furthermore, based on Theorem \ref{Z_k:S_k}, Theorem \ref{D2k_S2k}, and similar results for $\localFSymmetric{k}$, we obtain:

\begin{align}
    g \cdot \rho(x) &= \left(g_1 \cdot \rho(x)[1:k_1], \quad g_2 \cdot \rho(x)[k_1+1:k_2] \quad \dots \quad g_L \cdot \rho(x)[u:l]\right) \nonumber \\
    &= \left( \rho(h_1 \cdot x[1:l_1]), \quad \rho(h_2 \cdot x[l_1+1:l_2]) \quad \dots \quad \rho(h_L \cdot x[l_{L-1}:n]) \right) \nonumber \\
    &= \rho(h \cdot x) \nonumber
\end{align}

for some $h = (h_1, h_2, \dots, h_L) \in G$ and appropriately chosen $l_1, l_2, \dots, l_{L-1}$. This aligns with the claim presented in equation \eqref{eq:equivalent_claim_product_groups}.

\end{proof}


\subsection{Proof of Theorem \ref{thm:TS}}
\thmTS*
\begin{proof}
    Let $\Delta_a = \max_{\tilde{a} \in \mathcal{A}} \tilde{a}^{\top} \mu^\star - a^{\top} \mu^\star$ denote the gap in expected reward of an arm $a \in \mathcal{A}$, and let $a^\star$ be the optimal arm (thus $\Delta_{a^\star} = 0$). Let us define the LinTS algorithm's {\em cumulative} regret over $T$ rounds as $R_T = \sum_{a \in \mathcal{A}} \Delta_a \mathbb{E}\left[ N_T(a) \right]$, where $N_T(a) = \sum_{t=1}^T \mathbf{1} \left\{ a^{(t)} = a\right\} $ denotes the total number of times action $a$ is played in the time horizon $1, 2, \ldots, T$, and its {\em simple} regret for the guessed best arm after $T$ rounds as $R_T^{\text{simp}} = \mathbb{E}\left[ \Delta_{A_T} \right]$.

    By a standard result \cite[Prop. 33.2]{lattimore2020bandit} relating the simple regret to the cumulative regret, when the guessed arm $A_T$ is drawn according to the empirical distribution of plays as hypothesized, we have 
    \begin{equation}
    \label{eq:simp_cumu_reg}
        R_T^{\text{simp}} = \frac{R_T}{T}.   
    \end{equation}

    We can also bound the simple regret from below as 
    \begin{equation}
        \label{eq:lb_simp_reg}
        R_T^{\text{simp}} \geq \Delta_{\min} \, \mathbb{P}\left[ A_T \neq a^\star \right],
    \end{equation}
    where $\Delta_{\min} = \min \{\Delta_a : a \in \mathcal{A}, \Delta_a > 0 \}$ denotes the gap between the highest and second-highest expected reward across the arms.

    It is also separately known \cite[Thm. 3]{tsuchiya2020analysis} that the cumulative regret of LinTS for a finite action set admits the upper bound
    \begin{equation}
        \label{eq:linTS_reg}
        R_T \leq \kappa \log(T),
    \end{equation}
    where $\kappa \equiv \kappa \left(\mathcal{A}, \mu^\star, \nu \right)$ is a quantity depending on the actions $\mathcal{A}$, true parameter $\mu^\star$ and algorithm parameter $\nu$. Putting together \eqref{eq:simp_cumu_reg}, \eqref{eq:lb_simp_reg} and \eqref{eq:linTS_reg}, we obtain
    \[ \mathbb{P}\left[ A_T \neq a^\star \right] \leq \frac{\kappa \log(T)}{T \Delta_{\min}} \equiv \frac{c \log(T)}{T },\]
    with $c = \frac{\kappa}{\Delta_{\min}}$, in the form as claimed. 
\end{proof}

\subsection{Proof of Theorem \ref{Z_k:S_k linear non-realisability}}
\thmlinearZkSk*
\begin{proof}
Consider a $\localFCyclic{k}$-invariant function $\psi$ defined as follows:
\begin{equation}
    \psi(x) \neq \psi(y) \text{ if } y \notin \mathcal{O}_{\localFCyclic{k}}(x).
\end{equation}
In other words, the above-defined function assigns a unique value to each orbit. Suppose $\psi = \phi \circ M$ for some $\localFSymmetric{k}$-invariant function $\phi$ and some linear transformation $M$. Since each orbit $\mathcal{O}_{\localFCyclic{k}}(x)$ has a unique value and $\big|\mathcal{O}_{\localFCyclic{k}}(x)\big| \leq k$, we have
\begin{equation}
    \big|\psi^{-1}\left(\{c\}\right)\big| \leq k \quad \text{for any } c \in \text{Im}(\psi).
    \label{Ineq1}
\end{equation}
The linear transformation $M$ has a trivial null space, indicating that it has full rank and is bijective. Let $z \in \text{Im}(M)$ be such that all of its individual scalar components are unique. Such a vector exists in $\text{Im}(M)$ because $M$ is full rank, i.e.,
$$Mx = z$$
for some $x \in \mathbb{R}^k$. Then,
\begin{equation}
    \Big|\mathcal{O}_{\localFSymmetric{k}}(z)\Big| = k!.
\end{equation}
Since $k \geq 3$, we have $k! > k$. Thus, from \eqref{Ineq1}, we can see that this leads to a contradiction.
\end{proof}

\newpage
\bibliography{main}
\end{document}